\documentclass[letterpaper,twocolumn]{article} 
\usepackage{times}  
\usepackage{helvet}  
\usepackage{courier}  
\usepackage{url}  
\usepackage{graphicx}  
\usepackage{amsmath}
\usepackage{amsfonts}
\usepackage{amssymb}
\usepackage{amsthm}
\usepackage{comment}

\usepackage{bmpsize}
\usepackage[colorinlistoftodos]{todonotes}
\usepackage{tikz}

\newtheorem{theorem}{Theorem}[section]
\newtheorem{corollary}{Corollary}[theorem]
\newtheorem{lemma}[theorem]{Lemma}

\theoremstyle{definition}
\newtheorem{definition}{Definition}[section]

\usepackage{booktabs} 
\usepackage{multirow}
\usepackage{caption}
\usepackage{subcaption}

\usepackage{listings}
\usepackage[figure]{algorithm2e}

\begin{document}
\title{A Capacity Scaling Law for Artificial Neural Networks}

\author{Gerald Friedland\footnote{University of California, Berkeley and Lawrence Livermore National Lab}, Mario Michael Krell\footnote{International Computer Science Institute, Berkeley. Both authors contributed equally to this paper.}\\
friedland1@llnl.gov, krell@icsi.berkeley.edu
} 

\date{September 5, 2018}

\maketitle

\begin{abstract}
We derive the calculation of two critical numbers predicting the behavior of perceptron networks. First, we derive the calculation of what we call the lossless memory (LM) dimension. The LM dimension is a generalization of the Vapnik--Chervonenkis (VC) dimension that avoids structured data and therefore provides an upper bound for perfectly fitting almost any training data. Second, we derive what we call the MacKay (MK) dimension. This limit indicates a 50\% chance of not being able to train a given function. Our derivations are performed by embedding a neural network into Shannon's communication model which allows to interpret the two points as capacities measured in bits. We present a proof and practical experiments that validate our upper bounds with repeatable experiments using different network configurations, diverse implementations, varying activation functions, and several learning algorithms. The bottom line is that the two capacity points scale strictly linear with the number of weights. Among other practical applications, our result allows to compare and benchmark different neural network implementations independent of a concrete learning task. Our results provide insight into the capabilities and limits of neural networks and generate valuable know how for experimental design decisions.
\end{abstract}

\section{Introduction}
Understanding machine learning, as opposed to using it as a black box, requires insights into the training and testing data, the available hypothesis space of a chosen algorithm, the convergence and other properties of the optimization algorithm, and the effect of generalization and loss terms in the optimization problem formulation. One of the core questions that machine learning theory focuses on is the complexity of the hypothesis space and what functions can be modeled. For artificial neural networks, this question has recently become relevant again as deep learning seems to outperform shallow learning. For deep learning, single perceptrons with a nonlinear, continuous gating function are concatenated in a layered fashion. Techniques like convolutional filters, drop out, early stopping, regularization, etc., are used to tune performance, leading to a variety of claims about the capabilities and limits of each of these algorithms (see for example~\cite{zhang2017}). Even though artificial neural networks have been popular for decades, understanding of the processes underlying them is usually based solely on anecdotal evidence in a particular application domain or task (see for example~\cite{morgan2012}).

In this article, we attempt to change this trend by analyzing and making measurable what could intuitively be called the intellectual capacity of a neural network. This is, quantifying which functions can be learned as function of the number of parameters of the model. We follow the notion that feed-forward neural networks, just like Hopfield networks, can be best understood as associative memory. Instead of memorizing the data, perceptron networks memorize a function of the data. That is, they associate given (noisy) input to trained input and then map that to a trained label.  A closer look at the error as a function of capacity then reveals that perceptrons go through two phase transitions, as indicated earlier by Wolfgang Kinzel~\cite{kinzel1998}, similar to the ones observed in the self-assembly of matter or the Ising model of ferromagnetism. As a result, it is impossible to make an artificial neural network that is sensitive to, but not disrupted by, new information once a certain threshold is reached. Our theoretical derivation, backed up by repeatable empirical evidence, shows the scaling of the capacity of a neural network based on two critical points, which we call lossless-memory (LM) dimension and MacKay (MK) dimension, respectively. The LM dimension defines the point of guaranteed operation as memory and the MK dimension defines the point of guaranteed 50\,\% forgetting, even for very high dimensional networks. The scaling of both points is upper bounded strictly linearly with the number of weights. 

\begin{figure*}
\centering
\includegraphics[width=0.7\textwidth]{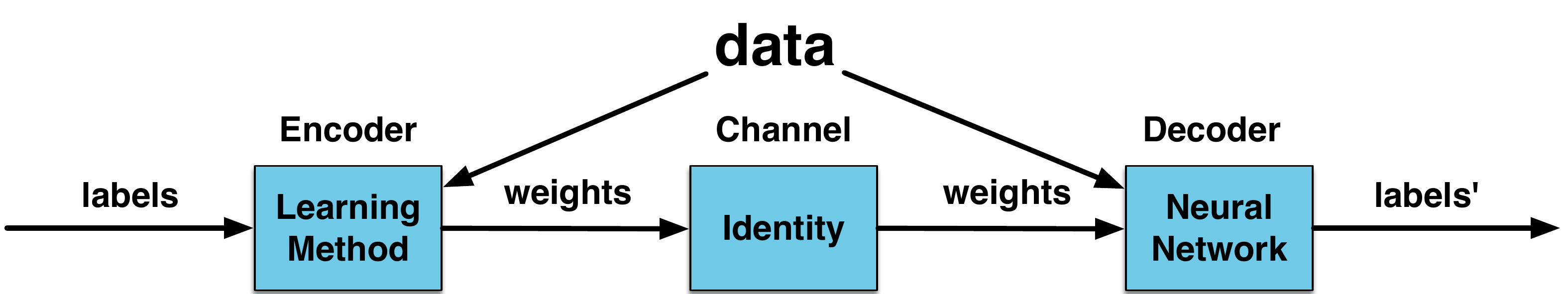}
\caption{\label{f:shannon}
Shannon's communication model applied to labeling in machine learning. A dataset consisting of $n$ sample points and the ground truth labeling of $n$ bits are sent to the neural network. The learning method converts it into a parameterization (i.\,e., network weights). In the decoding step, the network then uses the weights together with the dataset to try to reproduce the original labeling.}
\end{figure*}

\section{Related Work}
\label{sec:prior}
The perceptron was introduced in 1958~\cite{Rosenblatt1958} and since then has been extended in many variants, including but not limited to as described in~\cite{Crammer2006,Dekel2008,KrellPhd2015,KrellOc2015}. The perceptron uses a $k$-dimensional input and generates the output by applying a linear function to the input, followed by a gating function. The gating function is typically the identity function, the sign function, a sigmoid function, or the rectified linear unit (ReLU)~\cite{He2015,Nair2010}. Motivated by brain research~\cite{Feldman1982}, perceptrons are stacked together to networks and usually trained by chain rule (backpropagation)~\cite{Rumelhart1986,Rumelhart1988}. 

Even though perceptrons have been utilized for a long time, its capacities have been rarely explored beyond discussion of linear separability. Moreover, catastrophic forgetting has so far not been explained satisfactorily. Catastrophic forgetting~\cite{McCloskey1989,Ratcliff1990} describes the effect that when the net is first trained on one set of labels and then on another set of labels, it very quickly looses its capability to classify the first set of labels. Our interpretation is that one cause for this would be a capacity overflow in the second round of training.

One of the largest contributions to machine learning theory comes from Vapnik and Chervonenkis~\cite{Vapnik2000}, including the Vapnik-Chervonenkis (VC) dimension. The VC dimension has been well known for decades~\cite{Vapnik1971}. It is defined as the largest natural number of samples in a dataset that can be shattered by a hypothesis space.  This means that for a hypothesis space having VC dimension $D_{VC}$, there exists a dataset with $D_{VC}$ samples such that for any binary labeling ($2^{D_{VC}}$ possibilities) there exists a perfect classifier $f$ in the hypothesis space, that is, $f$ maps the samples perfectly to the labels. Due to perfect memorizing, it holds $D_{VC}=\infty$ for 1-nearest neighbor. Tight bounds have so far been computed for linear classifiers ($k+1$) as well as decision trees~\cite{Asian2009}. The definition of VC dimension comes with two major drawbacks, however. First, it considers only the potential hypothesis space but not other aspects like the optimization algorithm, or loss and regularization function that effect the choice of the hypothesis~\cite{Arpit2017}. Second, it is sufficient to provide only one example of a dataset to match the VC dimension. So given a more complex structure of the hypothesis space, the chosen data can take advantage of this structure. As a result, shatterability can be increased by increasing the structure of the data. While these aspects don't matter much for simple algorithms, it is a major point for deep neural networks. 

In~\cite{Vapnik1994}, Vapnik et al. suggest to determine the VC dimension empirically, but state in their conclusion that the described approach does not apply to neural networks as they are ``beyond theory". So far, the VC dimension has only been approximated for neural networks. For example, Mostafa argued loosely that the capacity must be bounded by $N^2$ with $N$ being the number of perceptrons~\cite{Mostafa1989}. Recently,~\cite{Shwartz2014} determined in their book that for a sigmoid activation function and a limited amount of bits for the weights, the loose upper bound of the VC dimension is $O(|E|)$ where $E$ is the set of edges and consequently $|E|$ the number of nonzero weights. Extensions of the boundaries have been derived for example for recurrent neural networks~\cite{Koiran1998} and networks with piecewise polynomials~\cite{bartlett1999almost} and piecewise linear~\cite{Harvey2017} gating functions. Another article~\cite{Koiran1997} describes a quadratic VC dimension for a very special case. The authors use a regular grid of $n$ times $n$ points in the two dimensional space and tailor their multilayer perceptron directly to this structure to use only $3n$ gates and $8n$ weights.

One measure that handles the properties of given data is the Rademacher complexity~\cite{Bartlett2001}. For understanding the properties of large neural networks, Zhang et al.~\cite{zhang2017} recently performed randomization tests. They show that their observed networks can memorize the data as well as the noise. This is proven by evaluating that their neural networks perfectly learn with random labels or with random data. This shows that the VC dimension of the analyzed networks is above the size of the used dataset. But it is not clear what the full capacity of the networks is. This observation also gives a good reason for why smaller size networks can outperform larger networks even though they have a lower capacity. Their capacity is still large enough to memorize the labeling of the data. A more elaborate extension of this evaluation has been provided by Arpit et al.~\cite{Arpit2017}. Our paper indicates the lower limit for the size of the network.

A different approach using information theory comes from Tishby~\cite{Tishby2015}. They use the information bottleneck principle to analyze deep learning. For each layer, the previous layers are treated as an encoder that compresses the data $X$ to some better representation $T$ which is then decoded to the labels $Y$ by the consecutive layers. By calculating the respective mutual information $I(X,T)$ and $I(T,Y)$ for each layer they analyze networks and their behavior during training or when changing the amount of training data. We describe the learning capabilities of neural networks using a different information theoretic view, namely the interpretation of neurons as memory cells. 

We are aware of recent questioning of the approach of discussing the memory capacity of neural networks~\cite{Arpit2017,zhang2017}. However, Occam's razor~\cite{blumer1987} dictates to follow the path of least assumptions and perceptrons were initially conceived as a "generalizing memory", as detailed for example, in the early works of  Widrow~\cite{widrow1962}. This approach has also been suggested by~\cite{Mostafa1989} and later explained in depth by MacKay~\cite{mackay2003}. In fact, initial capacity derivations for linear separating functions have already been reported by Cover~\cite{cover1965}. Also, the Ising model of ferromagnetism, which is clearly a model used to explain memory storage, has already been reported to have similarities to perceptrons~\cite{gardner1987,gardner1988} and also the neurons in the retina~\cite{tkacik2006ising}. 

\section{Capacity of a Perceptron}
\label{sec:mackay}
MacKay is the first one to interpret a perceptron as an encoder in a Shannon communication model (\cite{mackay2003}, Chapter~40). In our article, we use a slightly modified version of the model depicted in Fig.~\ref{f:shannon}.  We summarize his proof appearing in this section. The following definitions will be required.

\begin{definition}[VC Dimension~\cite{Vapnik2000}]
\label{def:vcd}
The VC dimension $D_{VC}$ of a hypothesis space $f$ is the maximum integer $D=D_{VC}$ such that \textit{some dataset} of cardinality $D$ can be shattered by $f$. Shattered by $f$ means that any arbitrary labeling can be represented by a hypothesis in $f$. If there is no maximum, it holds $D_{VC}=\infty$.
\end{definition}

\begin{definition}[General Position~\cite{mackay2003}]
    ``A set of points  $\{x_n\}$ in K-dimensional space are in general position 
    if any subset of size $\leq K$ is linearly independent, 
    and no $K + 1$ of them lie in a $(K −1)$-dimensional plane.''
\end{definition}

MacKay interprets a perceptron as an encoder in a Shannon communication model~\cite{shannon1948bell} (compatible to our interpretation in Fig.~\ref{f:shannon}). The input of the encoder are $n$ points in general position and a random labeling. The output of the encoder are the weights of a perceptron. The decoder receives the (perfectly learned) weights over a lossless channel. The question is then: Given the received set of weights and the knowledge of the data, can the decoder reconstruct the original labels of the points? In other words, the perceptron is interpreted as memory that stores a labeling of $n$ points relative to the data and the question is how much information can be stored by training a perceptron. In other words, we ask about the memory capacity of a perceptron. This communication definition not only has the advantage that the mathematical framework of information theory can be applied to machine learning, it also allows to predict and measure neural network capacity in the actual unit of information, bits.

The functionality of a perceptron is typically explained by the XOR example (i.\,e., showing that a perceptron with $2$ input variables, which can have $4$ states, can only model $14$ of the $16$ possible output functions). XOR and its negation cannot be linearly separated by a single threshold function of two variables and a bias. For an example of this explanation, see~\cite{Rojas1996}, section 3.2.2. MacKay effectively changes the computability question to a labeling question by asking: Given $n$ points, how many of the $2^n$ possible labelings in $\{0,1\}^n$
can be learned by the model without an error (rather than computing binary functions of $k$ variables). Just as done by~\cite{cover1965,Rojas1996}, MacKay uses the relationship between the input dimensionality of the data $k$ and the number of inputs $n$ to the perceptron, which is denoted by a function $T(n,k)$ that indicates the number of ``distinct threshold functions'' (separating hyperplanes) of $n$ points in general position in $k$ dimensions. The original function was derived by~\cite{schlaefli1852}. It can be calculated recursively as:

\begin{equation}
\label{eq:tnk}
T(n,k)=T(n-1,k)+T(n-1,k-1),
\end{equation}
where $T(n,1)=T(1,k)=2$ or iteratively:
\begin{equation}
\label{eq:tnk2}
T(n,k)=2\sum_{l=0}^{k-1}\genfrac(){0pt}{0}{n-1}{l}
\end{equation}

Namely, 
\begin{equation}
T(n,k)=2^n \text{ for } k\geq n.
\end{equation}
This allows to derive the VC dimension for the case $k=n$ where the number of possible binary labelings for $n$ points is $2^n$. Since $k=n$ and $T(n,n)=2^n$, all possible labelings of the input can be realized. 

When $k<n$, the $T(n,k)$ function follows a calculation scheme based on the Pascal Triangle~\cite{coolidge1949story}, which means that the 
bit loss due to incomplete shattering is still highly predictable. MacKay uses an error function based on the cumulative distribution 
of the standard Gaussian to perform that prediction and approximate the resulting distribution. More importantly, he defines a second point, which we call MK dimension. The MK dimension describes the largest number of samples such that typically only about $50\,\%$ of all possible labelings can be separated by the binary classifier. He proofs this point to be at $n=2k$ for large $k$ and illustrates that there is a sharp continuous drop in performance at this point. Since the sum of two independent normally distributed random variables is normal, with its mean being the sum of the two means, and its variance being the sum of the two variances, it is only natural that we will see in the following section that the MacKay point is linearly additive in the best case.

MacKay concludes that the capacity of a perceptron is therefore $2k$ as the error before that point is small. We follow Kinzel's physical interpretation~\cite{kinzel1998} and understand that the perceptron error function undergoes two phase transitions: A first order transition at the VC dimension and a continuous one at the MK dimension. Based on this interpretation, we predict that the different phases will play a role on structuring and explaining machine learning algorithms. We will therefore, throughout this paper, discuss the two points separately.

When comparing and visualizing $T(n,k)$ functions, it is only natural to normalize function
values by  the number of possible labelings $2^n$ and to normalize the argument by the  number
of inputs $k$ which is equal to the capacity of the perceptron. Figure~\ref{f:tnk} displays these normalized functions for different input dimensions $k$. The functions follows a clear pattern like the characteristic curves of circuit components in electrical engineering.

\section{Networks of Perceptrons}
\label{sec:combining}
For the remainder of this article, we will assume a feed-forward network. The weights are assumed to be real-valued and each unit has a bias, which counts as a weight. Note that no further assumptions about the architecture are required. Our derivations are upper bounds and therefore training-algorithm agnostic.

The definition of general position used in the previous section is typically used in linear algebra and is the most general case needed for a perceptron that uses a hyperplane for linear separation (see also Table~1 in~\cite{cover1965}). For neural networks, a stricter setting is required since neural networks can implement arbitrary non-linear separations. 

\begin{definition}[Random Position]
A set of points  $\{x_n\}$ in $K$-dimensional space is in random position, if and only if from any subset of size $<n$ it is not possible to infer anything about the positions of the remaining points.
\end{definition}

Note that random position implies general position, which was only excluding linear inference. Bear in mind that slightly distorted grid settings, as a minor modification of the example in~\cite{Koiran1997}, are in general position but not in random position. Random position is equivalent of saying that no inference is possible about the structure of the data and the only thing a machine learner can do is memorize. The only distribution that satisfies this constraint is the uniform distribution~\cite{gibbs1906}.   

As explained in Section~\ref{sec:prior}, it is possible to achieve very high VC dimension by the choice of very special datasets. This has not been an issue yet for learning theory but from a practitioner perspective, this has been criticized~\cite{shwartz2017,zhang2017,Arpit2017}. To avoid the reported problems and to be consistent with our embedding into the Shannon communication model, we therefore propose a generalization of the VC dimension which we call lossless memory dimension. 

\begin{definition}[Lossless Memory Dimension]
\label{def:led}
The lossless memory dimension $D_{LM}$ is the maximum integer number $D_{LM}$ such that for any dataset with cardinality $n\leq D_{LM}$ and points in \textit{random position}, \textit{all possible labelings} of this dataset can be represented with a function in the hypothesis space.
\end{definition}

Note that for a single perceptron $D_{LM}=D_{VC}$ because random position implies general position. As explained in Section~\ref{sec:prior}, we will name the corresponding point where loss is guaranteed MacKay dimension. 

\begin{definition}[MacKay Dimension]
\label{def:mkd}
The MacKay dimension $D_{MK}$ is the maximum integer $D_{MK}$ such that for any dataset with cardinality $n\leq D_{MK}$ and points in random position \textit{at least $50\%$} of all possible labelings of these datasets can be represented with a function in the hypothesis space \cite{mackay2003}.
\end{definition}

Consequently, a higher cardinality than $D_{MK}$ implies less than $50\%$ of the labelings can be represented. We will show that for an ideal perceptron network the limit is exactly $50\%$.

The proof becomes surprisingly easy, once one measures the memory capacity of each perceptron in bits~\cite{shannon1948bell}. In fact, it then becomes partly generalizable to any classifier treated as a black box. 

\subsection{Capacity}
Let us denote the lossless memory dimension of a binary classifier with $x$ parameters as $D_{LM}(x)$. Let $P$ be a set of points in random position. As usual, we denote as $|P|$ the number of points in the set. Furthermore, we denote as $|P|_{2}$ the number of bits used to represent these points. 

\begin{lemma}[Lossless Memory Dimension of Digital Classifiers]
\label{lem:maxlmd}
$D_{LM}(|P|) = |P|_2$
\end{lemma}
\begin{proof}
At $D_{LM}(|P|)$, by definition, all of the $2^{|P|}$ different labeling functions can be learned by the classifier. Since the points in $P$ are in random position, the classifier cannot learn any inference rule. Thus the pigeon hole principle implies that reproducing all possible labels requires $\log_2(2^{|P|})=|P|$ bits. Thus $D_{LM}(|P|) = |P|_2$ bits are required to to guarantee to be able to represent any of the $2^{|P|}$ equiprobable states.
\end{proof}

We note that $D_{LM}(|P|) > |P|_2$ can be contradicted easily as it implies universal lossless compression and cascading of such classifiers would allow to store and transfer any set of points $P$ with $1$ bit. 

Let $NN$ be a set of parameters for an arbitrary feed-forward perceptron network that shatters a set of points $P$ in random position. Let $D_{LM}(|NN|)$ be the classifiers' lossless memory dimension. Again, we denote as $|NN|_{2}$ the number of bits used to represent these parameters.

\begin{corollary}
\label{cor:maxlmdnn}
$D_{LM}(|NN|) = |NN|_2$
\end{corollary}
\begin{proof}
Lemma~\ref{lem:maxlmd} is universal to any binary classification model using digital weights. The special case for perceptron networks is therefore implied, $max(D_{LM}(|NN|)) = |NN|_2$. 
\end{proof}

The above already implies a linearly scaling upper bound in the number of bits used by the parameters. However, we are able to make this bound tighter with the following derivation. It turns out in a perceptron network, each parameter is only able to store one bit losslessly, independent of how many bits are used for the parameter.

Assume a perceptron with a set of parameters $PC$ that shatters a set of points $P$ in random position. Let $D_{LM}(|PC|)$ be it's lossless memory dimension. 

We already know from Section~\ref{sec:mackay} that $D_{LM}(|PC|) = |PC|$. However, we will provide an alternative proof here. Each perceptron uses a function $f$ of the form $f(x)={\begin{cases}1&{\text{if }}\ w\cdot x>b\\0&{\text{otherwise}}\end{cases}}$ where $w$ is a vector of real numbers and $b$ is a single real number. $w\cdot x$ is the dot product $\sum _{i=1}^{m}w_{i}x_{i}$.

\begin{lemma}[Lossless Memory Dimension of a Perceptron]
\label{lem:lmdperc}
$D_{LM}(|PC|) = |PC|$
\end{lemma}
\begin{proof}

Case 1: $b=0$\\
Let $b=0$. We now rewrite $\sum _{i=1}^{m}w_{i}x_{i}$ to 
$\sum _{i=1}^{m}s_{i}|w_{i}|x_{i}$, where $|w_{i}|$ is the absolute value of $w_i$ and $s_i$ is the sign of $w_i$, this is $s_i \in \{-1,1\}$. 

It can be easily seen that, given an $x_i$, the choice of $s_i$ is the determining factor for the outcome of $f$. $w_i$ merely serves as a scaling factor on the $x_i$. We also know from the proof of Lemma~\ref{lem:maxlmd} that the magnitude of $x$ does not matter.  

Since $s_i \in \{-1,1\}$ and $|\{-1,1\}|=|\{0,1\}|=2$ it follows that each $s_i$ can be encoded using $\log_2 2=1$ bit. With $w_i$ and $x_i$ irrelevant, we can therefore assume $|PC|_2=i$. It follows $|PC|_2=|PC|$. With $D_{LM}(|PC|) = |PC|_2$ (Corollary~\ref{cor:maxlmdnn}), it follows that $D_{LM}(|PC|) = |PC|$.

Case 2: $b\neq 0$\\
Using the same trick as above, we can write $b=s|b|$, where $|b|$ is the absolute value of $b$ and $s$ is the sign of $b$, this is $s \in \{-1,1\}$. We can now divide the $w_i$ by $|b|$ and obtain $\sum _{i=1}^{m}s_{i}\frac{|w_{i}|}{|b|}x_{i}$ and consequently change the first case of $f$ to $w_m\cdot x>s$ ($w_m$ denotes $w$ modified as explained). Since $s$ is not dependent on $i$ and $P$ is in random position, $s$ can in the general case only be trained to correct the decision of one $x \in P$. This is again because $s \in \{-1,1\}$ and thus $|\{-1,1\}|=|\{0,1\}|=2$ it follows that $s$ encodes $\log_2 2=1$ bit. In analogy to case 1, $|PC|_2=i+1=|PC|$ and $D_{LM}(|PC|) = |PC|_2$ (Corollary~\ref{cor:maxlmdnn}), it follows that $D_{LM}(|PC|) = |PC|$.
\end{proof}

This upper bound can now be generalized to a network of perceptrons. 

\begin{theorem}[Lossless Memory Dimension of a Neural Network]
\label{the:lmdnn}
$D_{LM}(|NN|) = |NN|$
\end{theorem}
\begin{proof}
Assume a perceptron network composed of $i$ perceptrons each with a set of parameters ${PC}_i$. It follows that the set of parameters in the neural network is $NN=\cup PC_i$. This is $|NN|=\sum |{PC}_i|$. We know from Lemma~\ref{lem:lmdperc} that $D_{LM}(|PC|) = |PC|$. This is an upper bound so adding bits to the same number of parameters in one perceptron has no effect. As a consequence, $D_{LM}(\sum |{PC}_i|)=\sum |{PC}_i|$. By simple substitution it follows that $D_{LM}(|NN|) = |NN|$.
\end{proof}

We note that this proof is consistent with MacKay's interpretation. Each perceptron with weights $k$ (including bias) is able to implement exactly $T(n,k)$ different binary threshold functions over $n$ sample points. With the maximum number of binary labelings of $n$ points being $2^n$, it follows that the perceptron is at $D_{LM}$ when $T(n,k)=2^n$. It is then able to maximally store $n$ bits. In general, adding two lossless memory cells with capacity $n$ and $m$ increases their capacity to $n+m$ bits. For lossy memory cells this is an upper limit -- which is all we are interested in. 

We also note that a shape imposed on the activation function does not play any role in theory: it is merely data processing on a decision made by the inequality.  

\begin{theorem}[MacKay Dimension of a Neural Network]
\label{the:mknn}
$D_{MK}(|NN|) = 2|NN|$
\end{theorem}
\begin{proof}
Assume a network at $D_{LM}(|NN|)=|P|$. Now let $P_2$ be a set of points in random position with $|P_2|=2|P|$. As discussed in Section~\ref{sec:mackay}, $T(2k,k)=\frac{1}{2}T(2k,2k)$. This is, doubling $|P|$ for a fixed $k$ results in each neuron being able to memorize the labeling of half of all points. As this is an upper limit, each perceptron can maximally equally contribute to the labeling of the additional points. It follows that $D_{Mk}=2 D_{LM}=2|NN|$.
\end{proof}

\begin{corollary}[Capacity Scaling of Perceptron Networks]
\begin{equation} \label{eq:cle}
\sum\limits_{j=1}^l C(P_j) = C\left(\sum\limits_{j=1}^l P_j\right)
\end{equation}
where $P_j$ is an arbitrary perceptron with $n_j$ inputs including a potential offset weight. The capacity $C$ is either $C=D_{MK}$ or $C=D_{LM}$ depending on the targeted phase. $\sum P_j$ denotes a neural network that combines the respective perceptrons perfectly and the data points are assumed to be in random position. 
\end{corollary}

This is, the upper bound of neural networks scales linearly in the amount of parameters. Practically, Equation~\ref{eq:cle} is an inequality ``$\geq$'' when the data is not in random position because the network should be able to exploit redundancies. On the other hand, many neural network implementations we measured turned out to be not maximally efficient (see Chapter~\ref{sec:experiments}).  

\section{Capacity Measurements}
\label{sec:experiments}
\label{s:mlp}
This section describes our evaluation of LM and MK dimension using empirical means. We observe that our theoretical capacities are indeed upper limits. 

\begin{figure*}[t!]
\centering
\includegraphics[width=0.49\textwidth]{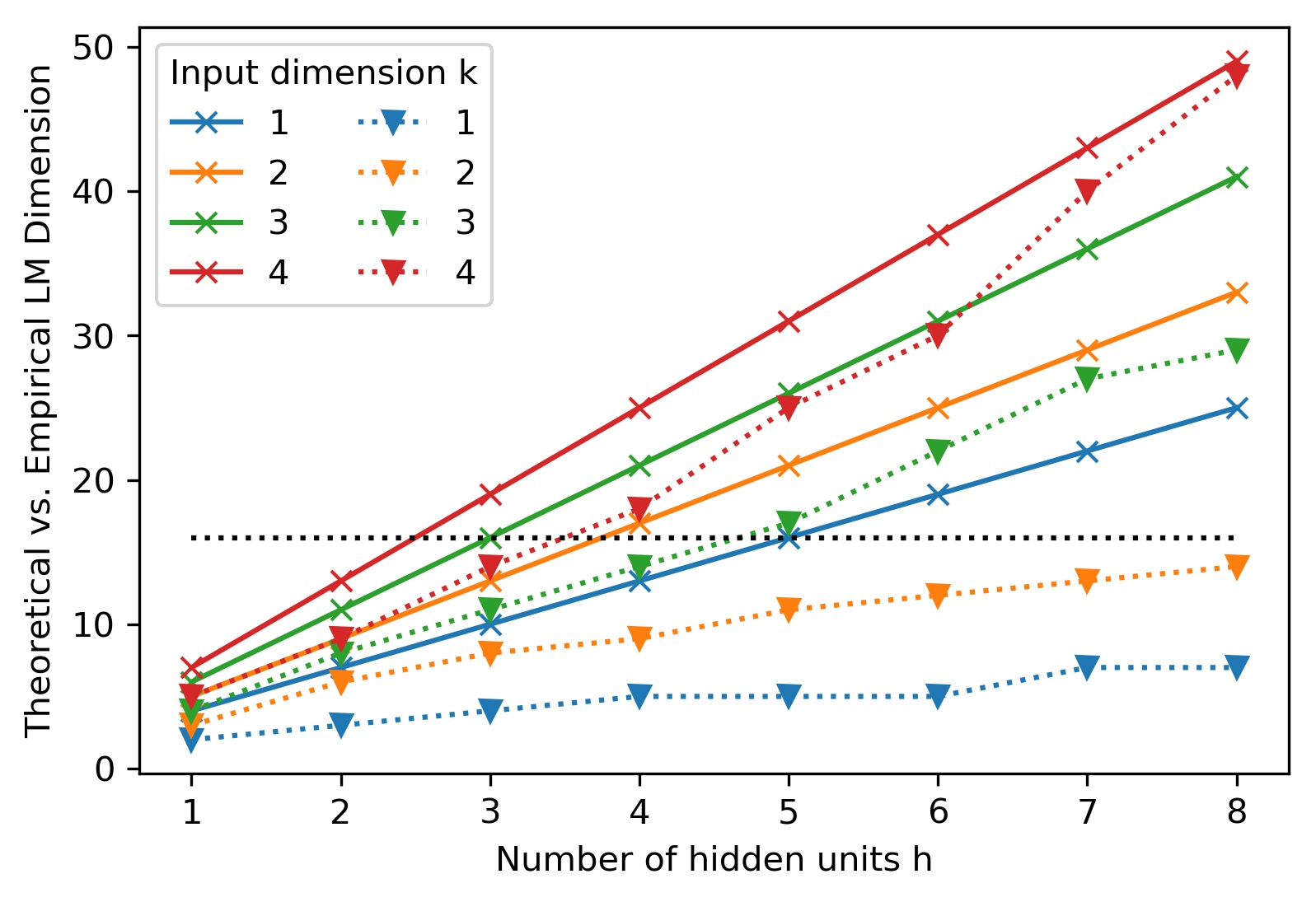}
\includegraphics[width=0.49\textwidth]{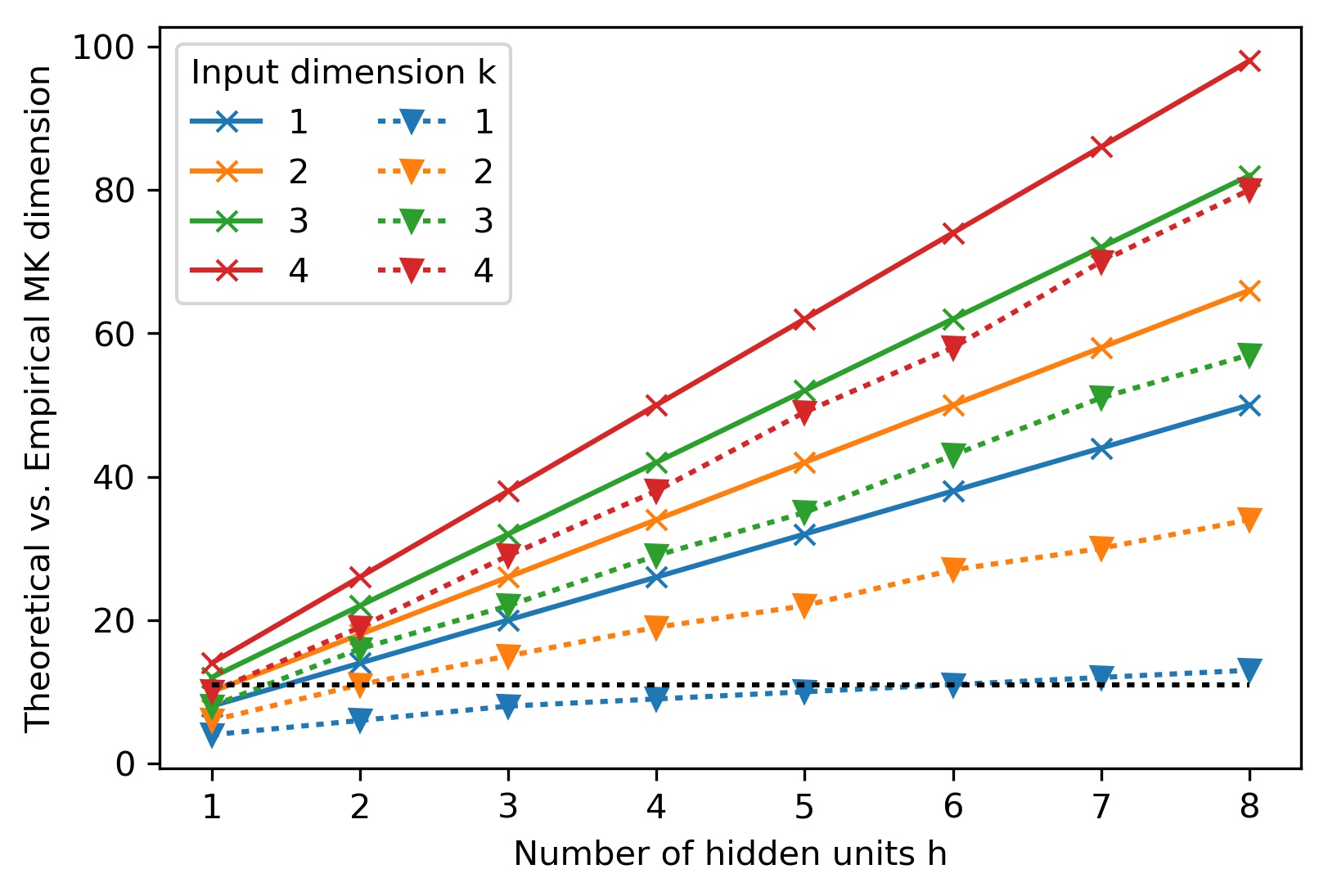}
\caption{\label{f:EVC}
Experimental results for LM dimension (left) and MK dimension (right). Displayed are the functional dependency on k (top) and on h (bottom).
The solid lines depict the theoretical boundaries whereas the respective dotted lines display our empirical results. The black lines display the number of samples where not all labelings are tested anymore but a random sample, which makes the empirical results less reliable.}
\end{figure*}

\subsection{Experimental Setup}
\label{sec:expsetup}

The basic principle for our empirical evaluation is to obtain samples from randomly generated data and increase the number of input points to the network step-by-step to test if the network can learn all possible labelings for the LM dimension or half of the possible labelings for the MK dimension. 

Obviously, we expect our empirical measurements to be lower than the theoretical capacities. Practically, neither the ideal network nor the perfect training algorithm exists. Furthermore, for higher dimensions, we were only able to sample from the hypothesis space and could not test  all labelings exhaustively. Therefore our goal was to create the best conditions possible and give the network the highest chance of reaching optimal capacity without violating the constraints of the theoretical framework. Thereby some practical workarounds are required for speedup and some limitations arise due to the exponential increase of the search space. 

We mainly used the MLP implementation in scikit-learn~\cite{Pedregosa2011} with L-BFGS~\cite{Liu1989} as optimizer. Our code is provided on the companion website to this article (see Section~\ref{sec:conclusion}). To control the randomness and ensure consistent results, we seed the randomizers with the respective index of the repetition. In case the optimizer does not fit the training data, we repeat its training up to $20$ times. Our data was randomly generated by sampling from a normal distribution. We repeated evaluations with up to $20$ different datasets if a labeling could not be fitted in the case of the LM dimension or if $50\%$ of the labelings could not be fitted in the case of the MK dimension. The processing time of the latter is much higher for two reasons. First, a larger amount of samples has to be analyzed since at least $50\%$ of all labelings have to be evaluated every time. Second, with more data the convergence of the MLP takes more iterations.

For completeness, every labeling would have to be tested. Due to symmetry in the class handling by the MLP, a minor speedup was achieved by testing only labelings where the last sample was labeled with a "0" and not a "1". This was not possible for large LM dimensions. Testing more than $2^{15}$ labelings was computationally too expensive for us. Hence, for more than $15$ samples, we tested only a random selection of $2^{15}$ labelings. Due to this approximation, results might be above the true values for the given structure. The processing effort of the MK dimension is even worse and required to have a limit of $2^{10}$ samples. Given more resources, one could imagine a better approach where multiple random samplings are tested and the median result for the MK dimension and the worst result for the LM dimension is taken. We leave this as future work.

The number of tested labels also limits the possible dimensions of the MLP. We analyzed input dimensions: $[1,2,3,4]$. $1$ did not provide reliable results. For the number of hidden nodes, we looked at $[1,2,3,4,5,6,7,8]$. Our implementation does not consider the difficulties of an MLP with class imbalance or redundancies. Here, higher empirical dimensions due to oversampling might be achievable.



\begin{figure*}
\centering
\includegraphics[width=0.49\textwidth]{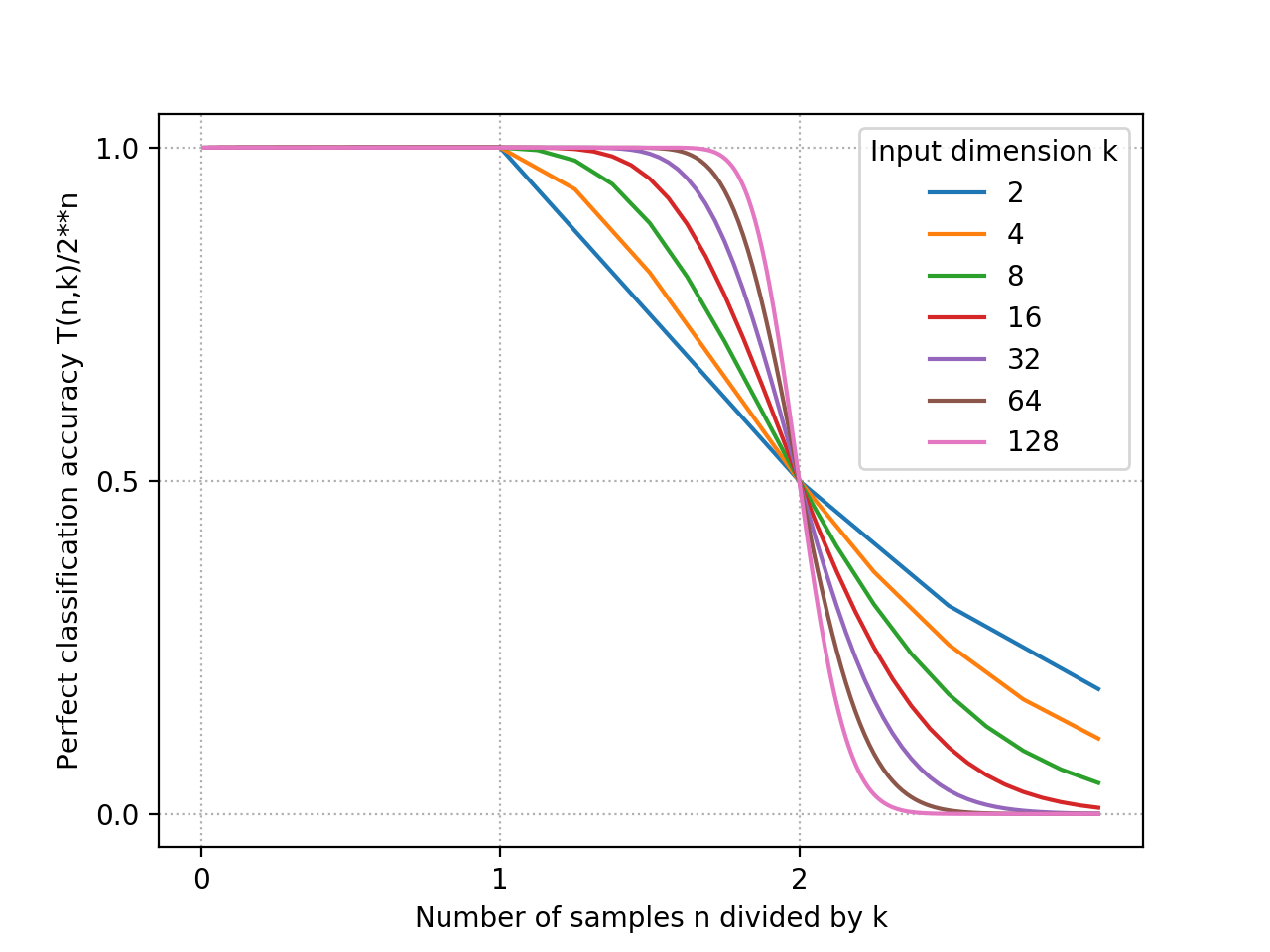}
\includegraphics[width=0.49\textwidth]{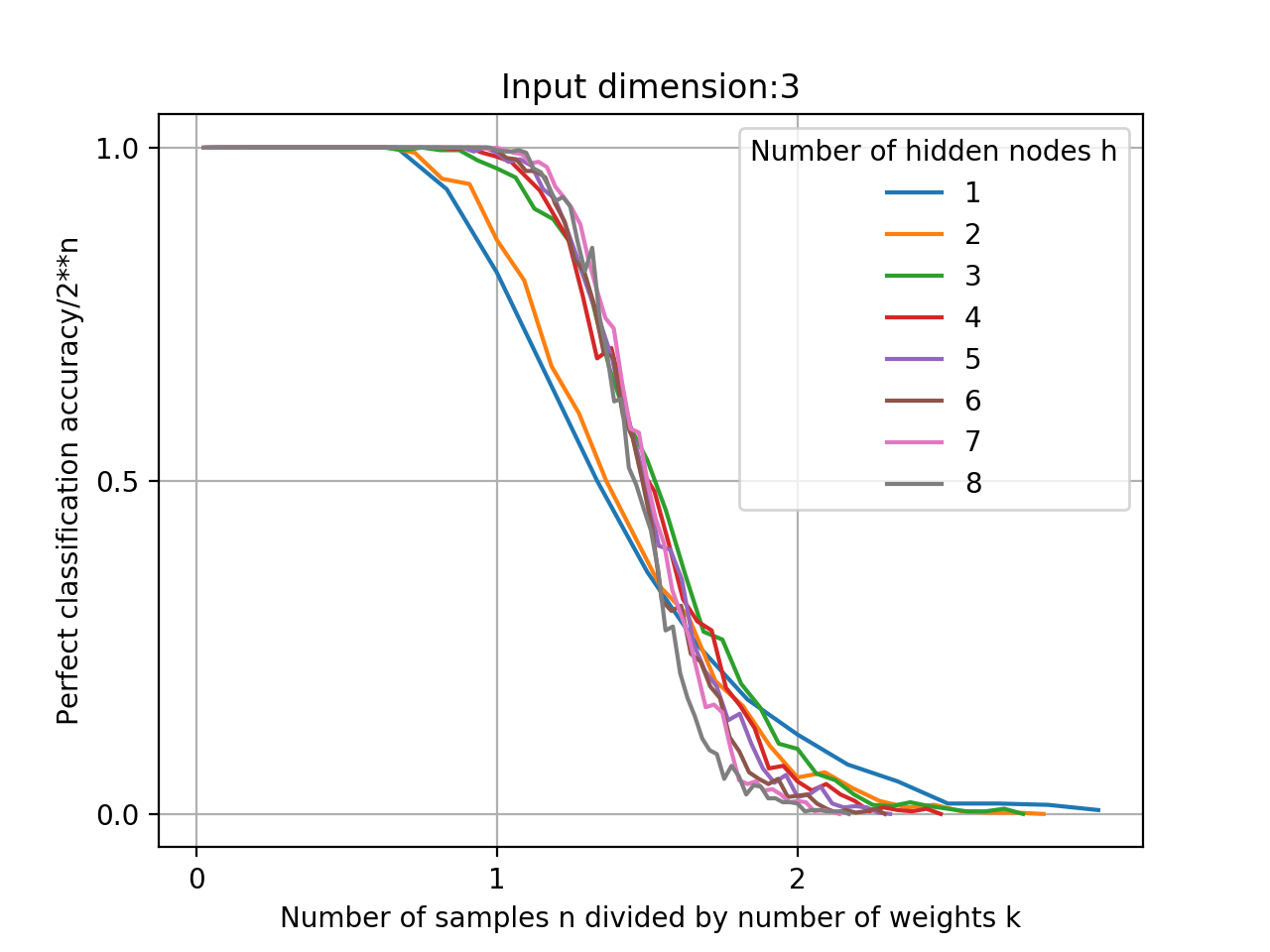}
\caption{
Left: \label{f:tnk}Characteristic curve examples of the $T(n,k)$ function for different input dimension $k$
with monotonic behavior and the two crucial points at $n=k$ for the VC dimension and $n=2k$
for the MK dimension. Right: \label{f:realtnk} Measured characteristic curve example. $x=1$ being the theoretic LM dimension and $x=2$ being the theoretic MK dimension (see Section~\ref{s:chareval}).}
\end{figure*}

\subsection{Tuning the Implementation}
Apart from the aforementioned implementation, we tested all other optimizers like ``Adam'' and ``SGD'' as well as the Keras library~\cite{chollet2015keras}. 
In most cases, the net was \textbf{not} able to fit the data in contrast to using L-BFGS. Hence the measured dimensions were very low. 
This could be interpreted as generalization capability of Adam and SGD because the optimizer is avoiding overfitting. Note, that L-BFGS approximates the second order derivative which makes it more accurate but also computationally more expensive and prone to get stuck in local minima.

We also tested different gating functions. Using the identity function, the network mostly behaved like a single perceptron as expected. For $tanh$ and logistic function, results looked similar to the ReLU function but needed more repetitions and processing time.

As expected, the generation of the data had a significant impact on the results. Originally, we tested with uniformly sampled data. Changing it to sampling from a normal distribution improved our results dramatically (i.\,e., the empirically measured upper bound came closer to the theoretical). The number of different tested datasets using the same distribution had only a minor effect on when the empirical calculation reached its limit in LM or MK dimension. The testing of more than one dataset was solely to capture the randomness in the training algorithm and had no significant impact on the empirical results. 

Using just one hidden neuron behaves always like a single perceptron with LM dimension $k+1$ and MK dimension $2(k+1)$. The predicted linear relationship in the number of hidden neurons $h$ as well as in $k$ for both dimensions can also be observed. The comparison between theoretical and empirical LM dimensions shows a similar linear behavior. For the larger LM dimensions, the differences get smaller but this is probably due to sampling error. For the VC dimensions, it is more important to test all labelings because a single misclassification has an impact, whereas for the MK dimension this effect is less severe. This could be improved in the future with more processing power.

We observed that the empirical MK dimension is extremely close to twice the empirical VC dimension. This is expected from the theoretical derivations but considering the aforementioned practical shortcuts, the clarity of this result increases our confidence in the validation experiments.

The empirical values for the VC dimension come quite close to $hk+1$ for small numbers which is $2h$ off from the optimal value. By increasing the number of iterations and tested datasets, we also detected three special cases that are worth pointing out here. For an MLP with $2$ hidden nodes and input dimensions of $3$, $4$, or $5$, we found a dataset example of $9$, $11$, or $13$ samples respectively that could be shattered. In those cases, we tested all labelings. Those sample values are exactly one sample higher than $h(k+1)$ and therefore above the storage capabilities of the hidden layer. Hence, the output neuron is making a significant contribution to the resulting learning capabilities, as predicted by the memory capacity formulation in this paper. 
 
We also performed experiments with going deeper than one layer and, as expected, there was no more than linear increase in the capacity of the network. In fact, in case of using small $k$ and $h$, the obtained results were better by just one sample compared to the respective LM dimension with a one hidden layer architecture. In most cases, we observed that the LM dimension was actually far below the empirical values of a respective network with one hidden layer. This can be explained by the data processing inequality and is left for future work.

\label{s:chareval}
Keeping the characteristic curves in Figure~\ref{f:tnk} in mind, it is also interesting how the characteristic curves of real networks look like when scaling by the theoretic
LM and MK dimension. Therefore, we used a similar evaluation but with only one dataset, only up to $2^{10}$ labelings,
and $50$ repetitions for the MLP optimization. 
For those we calculated the percentage of correctly learned labelings.
The results are depicted in Figure~\ref{f:realtnk}. It can be clearly seen that non-ideal networks still follow the characteristic behavior as it can be proven for the $T(n,k)$ function. However, the true transition points for LM and MK dimension are slightly shifted to the left.

\section{Conclusion}
\label{sec:conclusion}
We present an alternative understanding of neural networks using information theory. We show that the information capacity of a perceptron network scales maximally linearly with the number of parameters. The main trick is to train the network with random points. This way, no inference (generalization) is possible and the best thing any machine learner can do is memorize. We then determine how many parameters a neural network needs to have to be able to reproduce all possible labelings given these random points as input. The result is an upper bound on the size of the neural network as real world data is never random. This is, the inference ability of the network will often allow it to use less parameters and, assuming a perfectly implemented network, using as many parameters as for the random point scenario would be over fitting. As a consequence, a network at a larger capacity than LM dimension is, theoretically speaking, a waste of resources. On the other hand, if one wants to guarantee that a certain function can be learned, this is the theoretical number of parameters to use. However, when practically measuring concrete neural networks implementations with varying architectures and learning strategies, we found that their effectiveness actually varies dramatically (always below the theoretical upper limit). While this effectiveness measurement is exponential in run time, it only needs to be performed on a small representative subnet as capacity scales linearly. Therefore capacity measurement alone allows for a task-independent comparison of neural network variations. While our work is an extension of the initial work by David MacKay, this article is the first to generalize the critical points to multiple perceptrons and derive a concrete scaling law. Our experiments show that linear scaling holds practically and our theoretical bounds are actionable upper bounds for engineering purposes. All the tested threshold-like activation functions, including sigmoid and ReLU exhibited the predicted behavior -- just as explained in theory by the data processing inequality. Our experimental methodology serves as a benchmarking tool for the evaluation of neural network implementations. Using points in random position, one can test any learning algorithm and network architecture against the theoretical limit both for performance and efficiency (convergence rate). Future work in continuation of this research will explore tighter bounds, for example architecture-dependent capacity. Estimating the capacity needed for a given data set and ground truth will be another line of research. 

A web demo showing how capacity can be used is available at: \url{http://tfmeter.icsi.berkeley.edu}. Our experiments are available for repetition at: \url{ https://github.com/multimedia-berkeley/deep_thoughts}

\section*{Acknowledgements}
This work was performed under the auspices of the U.S. Department of Energy by Lawrence Livermore National Laboratory under Contract DE-AC52-07NA27344. It was also partially supported by a Lawrence Livermore Laboratory Directed Research \& Development grants (17-ERD-096 and 18-ERD-021). IM release number LLNL-TR-736950. Mario Michael Krell was supported by the Federal Ministry of Education and Research (BMBF, grant no. 01IM14006A) and by a fellowship within the FITweltweit program of the German Academic Exchange Service (DAAD). Any findings and conclusions are those of the authors, and do not necessarily reflect the views of the funders. We want to cordially thank Ra\'ul Rojas for in depth discussion on the chaining of the $T()$ function. We also want to thank Alfredo Metere, Jerome Feldman, Kannan Ramchandran, Alexander Fabisch, Jan Hendrik Metzen, Bhiksha Raj, Naftali Tishby, Jaeyoung Choi, Friedrich Sommer and Andrew Feit for their insightful advise and Barry Chen and Brenda Ng for their support. 

\bibliographystyle{abbrv}
\bibliography{main}

\newpage
\clearpage
\begin{algorithm}[ht!]
\centering
\begin{lstlisting}[frame=single,framexrightmargin=-0.0\textwidth,
                   basicstyle=\footnotesize,language=Python]
#Python 2.7 code for measuring the LM dimension 
#with 1 hidden layer
N = 80  # Maximum number of samples
K = [1,2,3,4]  # Analyzed dimensions
# Analyzed numbers of hidden layers
H = [1,2,3,4,5,6,7,8]
# Maximum number of samples, 
max_l = 15  # for random labelings
import itertools
import numpy
import random
from sklearn.neural_network \
	import MLPClassifier
	
print('n', 'k', 'h', 'correct', 'rate')
for k in K:  # input dimension
  numpy.random.seed(0)
  for h in H:  # number of hidden layers
    numpy.random.seed(0)
    for n in range(N):  # dataset size
      n += 1 # We start with one sample.
      data_res = [] # Good results
      # first label is fixed to be zero
      l_len = min(n-1,max_l-1)
      # 20 different random datasets
      for r_data in range(20):
        numpy.random.seed(r_data)
        # normal distributed data
        data = numpy.random.normal(
        	size=[N,k])
        numpy.random.seed(0)
        true_results = 0
        for label_int in range(2**l_len):
          if max_l < n:
            label_int = \
              random.randint(0, 2**(n-1))
          labels = [int(i) for i in bin(
            label_int*2+2**(N+2))[-n:]]
          d = data[:n]
          converged = False
          # repeated runs till converged
          for r_mlp in range(20): 
            clf = MLPClassifier(
              hidden_layer_sizes=(h,), 
              random_state=r_mlp,
              activation='relu', 
              solver='lbfgs', alpha=0)
            clf.fit(d, labels)
            p = clf.predict(d)
            if (p == labels).all():
              true_results += 1
              converged = True
              break  # short converged
          if not converged:
            break  # shortcut after miss
        data_res.append(true_results)
        # All labelings correct?
        if true_results == 2**l_len:
          break
      true_results = max(data_res)
      print(n, k, h, true_results, 
        true_results*1.0/2**l_len)
      if true_results*1.0/2**l_len<0.95:
          break
\end{lstlisting}
\label{c:VC}
\end{algorithm}
\begin{algorithm}[ht!]
\centering
\begin{lstlisting}[frame=single,framexrightmargin=-0.0\textwidth,
                   basicstyle=\footnotesize,language=Python]
#Python 2.7 code for measuring the MK dimension 
#with 1 hidden layer
N = 120  # Maximum number of samples
K = [1,2,3,4]  # Analyzed dimensions
# Analyzed numbers of hidden layers
H = [1,2,3,4,5,6,7,8]
# Maximum number of samples, 
max_l = 10  # for random labelings
import itertools
import numpy
import random
from sklearn.neural_network \
	import MLPClassifier
print("n","k","h","correct", "rate")
for k in K:  # input dimension
  numpy.random.seed(0)
  for h in H:  # number of hidden layers
    numpy.random.seed(0)
    for n in range(N):  # dataset size
      n += 1  # We start with one sample
      if n <= 2*(h)*(k-1)+k+1:
        continue  # shortcut
      data_res = []  # Good results
      # first label is fixed to be zero
      l_len = min(n-1,max_l-1)
      # 20 different datasets
      for r_data in range(20):
        numpy.random.seed(r_data)
        data = numpy.random.normal(
        	size=[N,k])
        numpy.random.seed(0)
        true_results = 0
        for label_int in range(2**l_len):
          index = label_int
          if max_l < n:
            label_int = \ 
              random.randint(0, 2**(n-1))
          labels = [int(i) for i in bin(
            label_int*2+2**(N+2))[-n:]]
          d = data[:n]
          converged = False
          for r_mlp in range(20):
            clf = MLPClassifier(
              hidden_layer_sizes=(h,), 
              random_state=r_mlp,
              activation='relu', 
              solver="lbfgs", alpha=0)
            clf.fit(d, labels)
            p = clf.predict(d)
            if (p == labels).all():
              true_results += 1
              converged = True
              break  # short converged
          # 50% labelings correct?
          if true_results>=2**(l_len-1):
            break  # short success
          if index-true_results > 2**(
              l_len-1):
            break  # short fail
        if true_results >= 2**(l_len-1):
          data_res.append(true_results)
          break  # short success
        data_res.append(true_results)
      true_results = max(data_res)
      print(n, k, h, true_results, 
        true_results*1.0/2**l_len)
      if true_results*1.0/2**l_len<0.45:
        break
\end{lstlisting}
\label{c:MK}
\end{algorithm}

\end{document}